\documentclass[journal]{IEEEtran}
\usepackage{amsmath,amsfonts}
\usepackage{amsmath}
\usepackage{algorithmic}
\usepackage{algorithm}
\usepackage{array}
\usepackage[caption=false,font=normalsize,labelfont=sf,textfont=sf]{subfig}
\usepackage{textcomp}
\usepackage{stfloats}
\usepackage{url}
\usepackage{verbatim}
\usepackage{graphicx}
\usepackage{cite}
\usepackage{bm}
\usepackage{amsthm}
\usepackage{tabularx}
\usepackage{graphicx}
\usepackage{subfig}
\usepackage{orcidlink}
\usepackage{soul}
\usepackage{color,xcolor}

\newtheorem{assumption}{Assumption}
\newtheorem{lemma}{\textbf{Lemma}}
\newtheorem{theorem}{Theorem}
\newcommand{\mathcolorbox}[2]{\colorbox{#1}{$\displaystyle #2$}}

\soulregister\cite7

\soulregister\ref7

\newcommand{\colorbibs}[2][orange]%
{%
\DeclareBibliographyCategory{ColoredBiblist#1}%
\addtocategory{ColoredBiblist#1}{#2}%
\AtEveryBibitem{\ifcategory{ColoredBiblist#1}{\color{orange}\bfseries}{}}
}
\newcommand\highlightReference[1]{%
  \expandafter\newcommand\csname highlightReference-#1\endcsname{}%
}
\let\oldbibitem\bibitem
\def\bibitem#1 #2\par{%
  \expandafter\ifx\csname highlightReference-#1\endcsname\relax
    \oldbibitem{#1}#2\par
  \else
    \oldbibitem{#1}\highlight{#2}\par
  \fi
}
\newcommand\highlight[1]{{#1}}

\begin{document}


\title{Prescribed Performance Control of  Deformable \\Object Manipulation in Spatial Latent Space}

\author{Ning Han~\orcidlink{0009-0007-4377-1461}, Gu Gong~\orcidlink{0000-0002-2025-5444},~\IEEEmembership{Student Member,~IEEE,} Bin Zhang~\orcidlink{0000-0001-6860-1951},~\IEEEmembership{Student Member,~IEEE,} Yuexuan Xu~\orcidlink{0000-0001-8792-537X},\\ Bohan Yang~\orcidlink{0000-0002-5305-7591}, Yunhui Liu~\orcidlink{0000-0002-3625-6679},~\IEEEmembership{Fellow,~IEEE,} and David Navarro-Alarcon~\orcidlink{0000-0002-3426-6638},~\IEEEmembership{Senior Member,~IEEE}

\thanks{This work is supported by the Research Grants Council of Hong Kong under grant AoE/E-407/24-N. \textit{Corresponding author: David Navarro-Alarcon.}}
\thanks{N. Han, G. Gong, B. Zhang, Y. Xu and D. Navarro-Alarcon are with the Department of Mechanical Engineering, The Hong Kong Polytechnic University, Kowloon, Hong Kong (e-mail: ningg.han@connect.polyu.hk; 22041178r@connect.polyu.hk; me-bin.zhang@connect.polyu.hk; yuexuan.xu@connect.polyu.hk;
dnavar@polyu.edu.hk)}
\thanks{B. Yang and Y. Liu are with the T Stone Robotics Institute, Department of Mechanical and Automation Eng., The Chinese University of Hong Kong, NT, Hong Kong (e-mail: bhyang@link.cuhk.edu.hk; yhliu@cuhk.edu.hk).}
}

\markboth{Han \MakeLowercase{\textit{et al.}}: Prescribed Performance Control of Deformable Object Manipulation in Spatial Latent Space}{}

\maketitle
\begin{abstract}
Manipulating three-dimensional (3D) deformable objects presents significant challenges for robotic systems due to their infinite-dimensional state space and complex deformable dynamics. This paper proposes a novel model-free approach for shape control with constraints imposed on key points. Unlike existing methods that rely on feature dimensionality reduction, the proposed controller leverages the coordinates of key points as the feature vector, which are extracted from the deformable object's point cloud using deep learning methods. This approach not only reduces the dimensionality of the feature space but also retains the spatial information of the object. By extracting key points, the manipulation of deformable objects is simplified into a visual servoing problem, where the shape dynamics are described using a deformation Jacobian matrix. To enhance control accuracy, a prescribed performance control method is developed by integrating barrier Lyapunov functions (BLF) to enforce constraints on the key points. The stability of the closed-loop system is rigorously analyzed and verified using the Lyapunov method. Experimental results further demonstrate the effectiveness and robustness of the proposed method.
\end{abstract}

\begin{IEEEkeywords}
Latent space, adaptive control, prescribed performance control, barrier Lyapunov function.
\end{IEEEkeywords}

\section{Introduction}

\IEEEPARstart{D}{eformable} object manipulation (DOM) of robots has emerged as a significant area of research due to its broad range of applications including medical surgery, industrial welding, and automated cloth folding. The dexterity of robotic manipulation is crucial for effectively managing deformable materials, enabling robots to contribute significantly across various sectors. Nonetheless, existing robotic control systems exhibit substantial limitations that hinder their capacity to achieve the advanced functionalities required for such tasks. Consequently, the development and investigation of robust and efficient control methodologies are imperative to bridge this gap and enhance the performance of robotic systems in DOM, thereby enabling their wider applicability and integration into complex real-world scenarios.

\begin{figure}[!t]\centering
	\includegraphics[width=\linewidth]{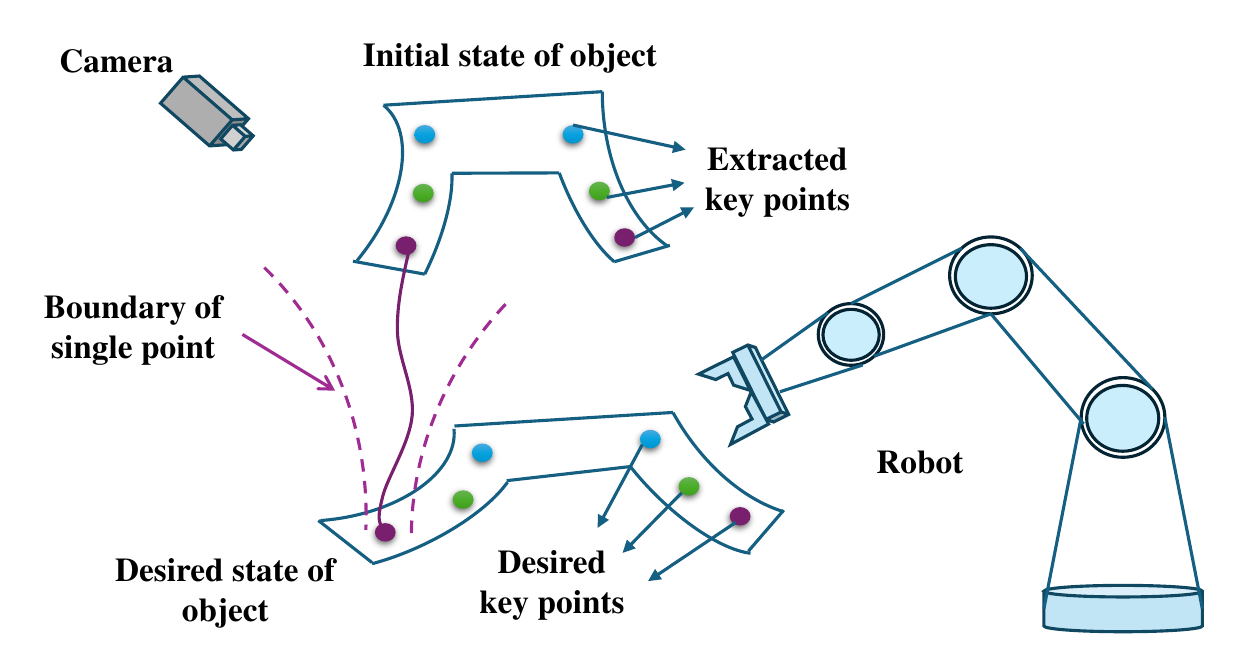}
	\caption{The configuration of 3D deformable object manipulation with constraints of key points.}
\label{domp}
\end{figure}

\subsection{Related Work}

DOM is currently a very {active} research topic in the robotic manipulation community. 
Existing methods for DOM can be primarily categorized into two approaches: model-free and model-based methods. Model-based DOM mainly relies on physical models to predict the deformation of deformable objects, such as the mass-spring model \cite{9729408}, finite element analysis method \cite{8593512}, etc. {Model-based methods are highly dependent on the accuracy of models and the ability to estimate its parameters accurately. However, due to difficulties in obtaining accurate parameters, recent works combine model-based approaches with some data-driven algorithms for sim-to-real control to reduce the reliance on model accuracy \cite{ccc}.}

Model-free DOM methods mainly include traditional data-driven methods and learning-based methods. Traditional data-driven methods use iterative learning methods, such as adaptive learning \cite{9888782,10122176} and Kalman filter \cite{1013418}, to approximate the deformation Jacobian matrix (DJM) of the object, which is then used to generate control signals through visual servoing control methods \cite{9634846}. {Learning-based methods mainly include model predictive control (MPC) methods and  reinforcement learning (RL) methods. MPC methods mainly utilize deep neural networks \cite{8769898}, graph neural networks \cite{10758319}, and Gaussian process regression \cite{8258951} to predict the deformation of objects, and then the learned object deformation model  will perform as a constraint for MPC in DOM control \cite{10378708}. RL methods gradually guide the robot to manipulate deformable objects to the target shape by setting the target reward and process reward \cite{cbsystems.0114, zhao2023learningfinegrainedbimanualmanipulation}. Recent work utilizes Vision-Language-Action to understand DOM through multimodal representation and perform end-to-end control \cite{black2024pi0visionlanguageactionflowmodel}.}

Nevertheless, both model-based and model-free approaches encounter significant challenges in the control of deformable objects due to their inherent high-dimensional state space. To address this challenge, the state of the deformable object is typically projected into a low-dimensional latent space. Especially, commonly employed techniques for this dimensionality reduction include Principal Component Analysis (PCA) \cite{ZHU2021103798}, B-spline \cite{10120758}, and autoencoders \cite{9928322}. These methods can effectively reduce the feature dimensionality and facilitate the design of the controller. However, the feature vectors obtained through these methods belong to an abstract latent space, lacking physical and spatial information. This limitation hinders the ability to utilize such information, which is critical for achieving precision in real-time control. {To solve this problem, several methods are proposed to avoid latent abstractions entirely and keep spatial and geometric information, such as Fourier series \cite{8106734}, Procrustes analysis \cite{10311090}, and Cosserat Model \cite{10654562}. However, these methods are limited by the two-dimensional structure, rigid assumptions, and strong model dependence, respectively, and cannot be widely used in various scenarios.}

As the capabilities and real-time performance of deep learning networks continue to improve, recent research has increasingly leveraged convolutional neural networks (CNNs) \cite{10.5555/3454287.3455249} and PointNet \cite{8099499} to efficiently compress the features of 2D images and 3D point clouds of objects. {Compared with traditional methodologies such as the Fourier series, unsupervised learning-based approaches possess the capability to directly extract key points from the surfaces of deformable objects using images or point clouds, without relying on prior knowledge. These sets of feature points effectively represent a latent space that encodes critical spatial information.} This capability retains the spatial information inherent in these features, ensuring that critical geometric details are preserved during the dimension reduction \cite{9811597}. The advantage of utilizing key points as the feature vector is that it facilitates the establishment of constraints on key point positions when designing controllers, enabling precise control of deformable objects based on these identified key points.

Jacobian-based prescribed performance control (PPC) is a widely used method for visual servoing control with constraints \cite{8720242,6942538}. This method confines visual feature errors within preset boundaries through the design of error boundaries and transferred errors \cite{10313035}, thereby enhancing both transient and steady-state control performance. {However, these methods are only applicable to visual servoing control where the Jacobian matrix is known. How to transfer these methods to DOM tasks where the Jacobian matrix cannot be obtained, so as to improve the control accuracy of DOM, is an important research topic.}

\subsection{Our Contribution}
{Inspired by \cite{9811597,9888782,10313035}}, this paper introduces a prescribed performance control strategy for DOM in latent space while incorporating spatial information. Specifically, a deep learning architecture termed Key-Grid \cite{5djsio} is employed to extract key points from the point cloud representation of deformable objects, with the 3D coordinates of these key points serving as feature descriptors. Subsequently, a Jacobian-based prescribed performance controller is developed, integrating a prescribed performance function to ensure that the errors of key points converge within the predefined performance bounds, while the DJM is approximated using a radial basis function neural network (RBFNN). {Compared with \cite{9811597}, which utilizes manually marked key points and designs an optimization controller based on an adaptive Jacobian matrix, and \cite{9888782}, which extracts key points from depth images and designs a graph network-based MPC controller, our method offers a distinct approach. Specifically, we extract key points directly from 3D point clouds and integrate them with an improved version of the PPC framework proposed in \cite{10313035}. We successfully migrated this type of PPC controller from the visual servoing tasks based on accurate Jacobian matrices obtained by hand-eye calibration to the DOM tasks where the Jacobian matrix is completely unknown.} The original contributions of this paper are summarized as follows:
\begin{itemize}
\item[$\bullet$] We develop a morphological presentation of deformable objects which utilizes Key-Grid neural network to {embed} the state of the deformable object into the latent space with spatial information.
\item[$\bullet$] We proposed a motion controller that integrates prescribed performance functions to constrain the spatial errors of key points, effectively improving accuracy.
\item[$\bullet$] We construct a Barrier Lyapunov function and spatial error boundaries to ensure stability of the closed-loop system, thereby guaranteeing the boundedness of the errors of the key points.
\item[$\bullet$] We conduct detailed experiments with ablative and comparative studies to evaluate the performance of our proposed algorithm. The results demonstrate that our method outperforms other approaches.
\end{itemize}

The rest of this article is organized as follows: Sec. II formulates the DOM problem. Sec. III presents the key point extraction and the control method, along with proof of the Lyapunov stability. Experimental results are provided in Sec. IV. 
{Finally, Sec. V summarizes the advantages and limitations of this work, and provides direction for improvement.}

\section{Problem Statement}
In this paper, we investigate a 3D DOM problem using a dual-arm robotic system, where each arm has six degrees of freedom. Specifically, this work focuses on manipulating sponges to perform various tasks, as illustrated in Fig. \ref{domp}.

All of these tasks can be regarded as a sequence of shape control tasks. For each shape control task, the problem is formulated as: Consider a deformable object which is represented by a set of 3D points $\mathbf{p}_{raw} \in \mathbb{R}^{3N}$ {where $N$ is an extremely high integer ($N \gg 10^5$)}, extract key points $\mathbf{p} \in \mathbb{R}^{3n}$ from this high-dimension vector, where $n$ denotes the number of key points with $ n \ll N$. Then, given the desired shape $\mathbf{p}_{raw}^* \in \mathbb{R}^{3N}$ and key points $\mathbf{p}^* \in \mathbb{R}^{3n}$, control the joint speeds $\dot{\mathbf q} \in \mathbb{R}^{12}$ to make the visual errors $\mathbf{e}_p =\mathbf  p - \mathbf p^* \in \mathbb{R}^{3n}$ converge to zero while satisfying prescribed constraints. 

Before introducing the method proposed in this paper, the following assumptions are made.
\begin{assumption}
    The object is rigidly grasped by the robot so that there is no relative displacement between the object and the gripper of the robot arm.
\end{assumption}
\begin{assumption}\label{slow}
    The robot manipulating motion is sufficiently slow such that we can utilize Jacobian matrix to formulate the dynamics of quasi-static elastic deformation of the object.
\end{assumption}

Moreover, to ensure clarity and maintain consistency throughout this paper, we adhere to the following notation conventions: scalar are represented by italicized lowercase letters (e.g., $a$), vectors are denoted by bold lowercase letters (e.g., $\mathbf{a}$), and matrices are indicated using bold uppercase letters (e.g., $\mathbf{A}$).

\section{Methodology}
{In this paper, we propose a DOM control method which utilizes a Key-Grid neural network to extract key points of deformable objects as features, combined with the PPC method to improve the steady-state and transient performance of the system. The overall structure of our proposed method is displayed in Fig. \ref{fig_str}.}

\begin{figure*}[]\centering
	\includegraphics[width=1\linewidth]{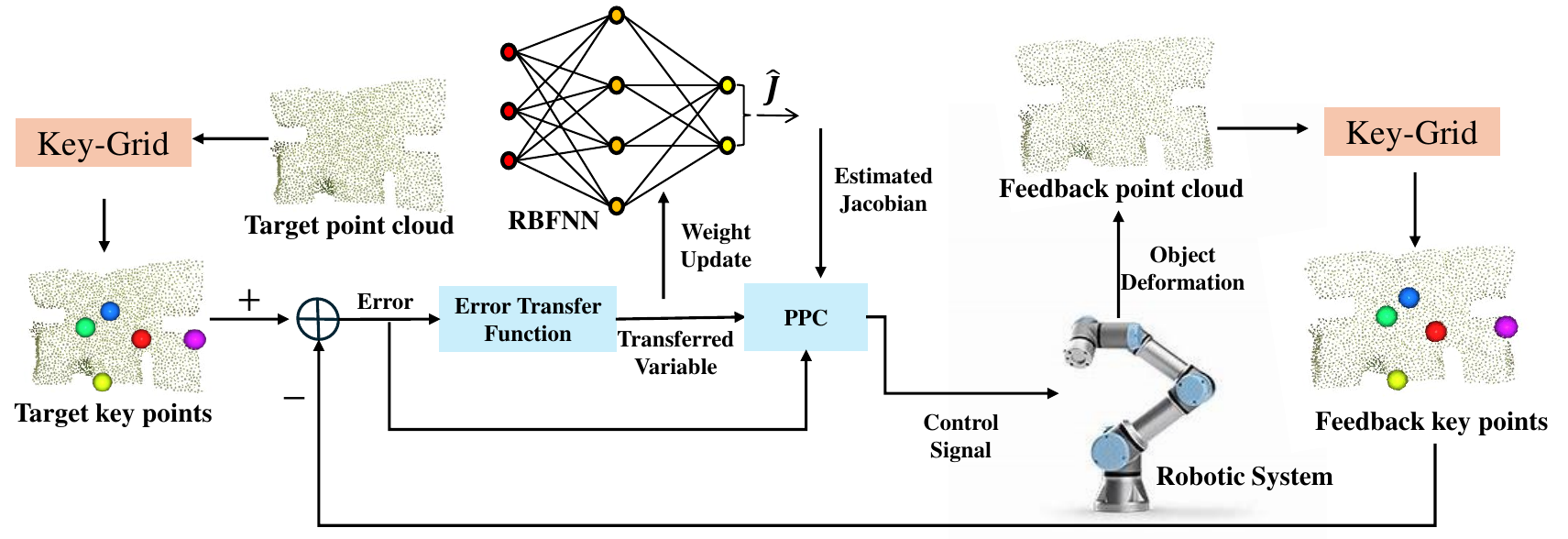}
	\caption{{Structure of our proposed method.}}
\label{fig_str}
\end{figure*}

\subsection{Key Points Extraction from Visual Observations}
The key idea of the proposed method is to extract some key points of the deformable objects as features. To avoid adding labels to deformable objects for training or control, we introduce an unsupervised learning method, Key-Grid whose structure is displayed in Fig. \ref{fig_2}, to extract key points. Precisely, a Key-Grid network consists of two components, encoder $f_{enc}$ and decoder $f_{dec}$. The encoder $f_{enc}$ consists of $L$ PointNet++ layers with a Softmax activation function applied to the final layer. Consequently, the coordinates of the predicted key points $\mathbf P \in \mathbb{R}^{n \times 3}$ can be extracted by the encoder.
\begin{figure*}[]\centering
	\includegraphics[width=1\linewidth]{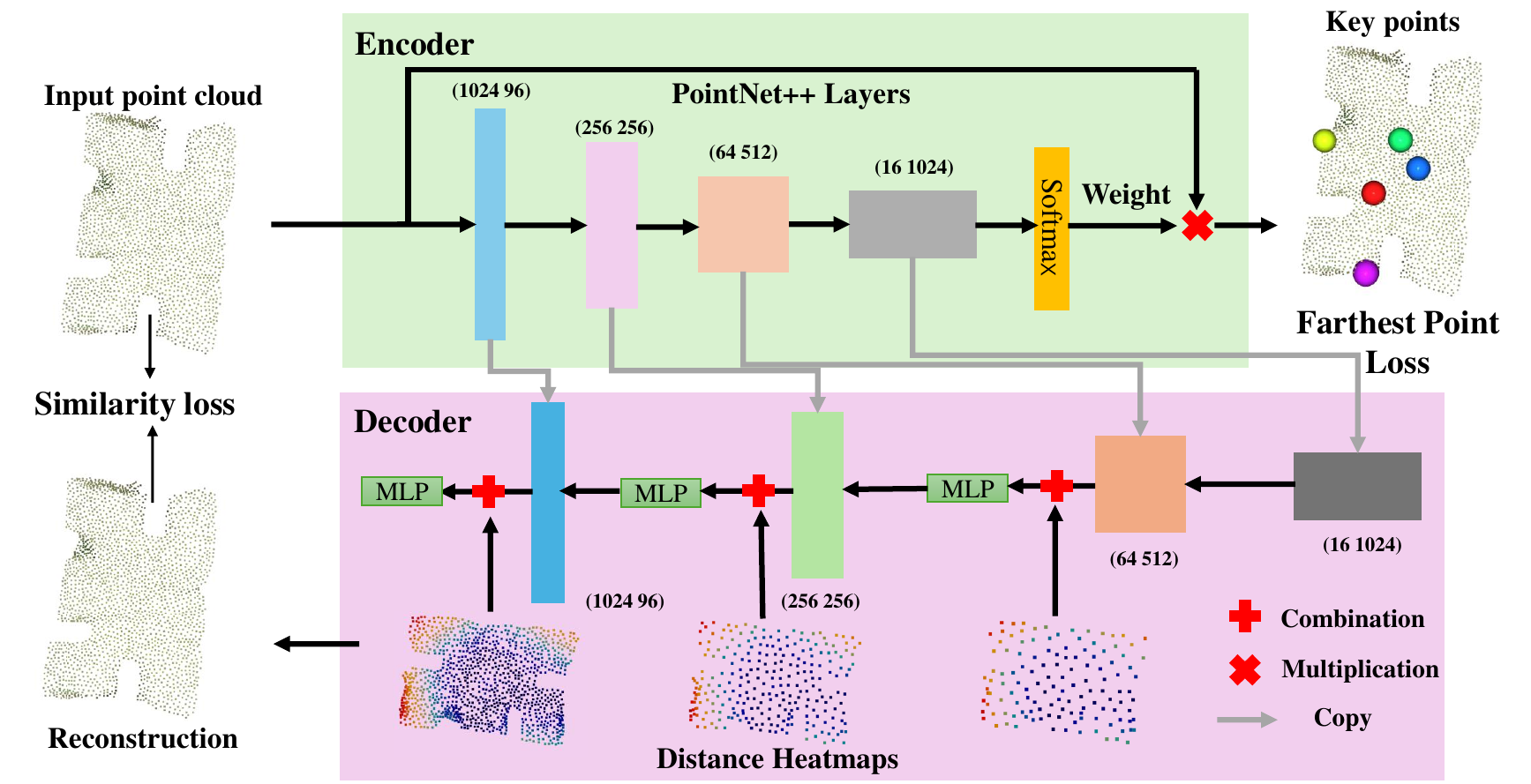}
	\caption{Structure of Key-Grid. In the encoder section, the key points are extracted from the input point cloud by utilizing the PointNet++. The detected key points are then used to form grid heatmaps. In the decoder section, each layer of the PointNet++, heatmaps, and MLPs are utilized to reconstruct the input. The farthest point loss of the key points and the reconstructed point cloud similarity loss are used for network training.}
\label{fig_2}
\end{figure*}
\begin{equation}
   \mathcolorbox{white}{ \mathbf P =  \mathbf {W}_{key} \cdot \mathbf {X}_{key}}
\end{equation}
where $ \mathcolorbox{white} {\mathbf X_{key} \in \mathbb{R}^{N \times 3 }}$ represents input point cloud and $\mathcolorbox{white}{ \mathbf W_{key} = \text{Softmax}(f_{enc}) \in \mathbb{R}^{n \times N}}$. Then, a skeleton is constructed by connecting each pair of predicted key points, and the feature of a point $d(p)$ can be defined as the maximum of the weighted distances from this point to the edges of the skeleton.
\begin{equation}
    d(p) = \mathop{\text{max}}\limits_{i,j}\left[ s_{ij}\ \text{e}^{(d_{ij}^2(p)/ \mathcolorbox{white}{\nu}^2\color{black})}  \right]
\end{equation}
where $d_{ij}(p)$ represents the distance between the point $p$ and the edge of the skeleton, which connects the predicted key points $\mathcolorbox{white}{ \mathbf k_i}$ and $\mathcolorbox{white}{\mathbf k_j}$, $s_{ij}$ refers to the learnable weight of this edge produced by the encoder and $\mathcolorbox{white}{\nu}$ is a hyper-parameter. {Thus, a grid heatmap of the point $\mathbf{h}(\mathbf{d}(p))$ can be generated, which provides a continuous geometric representation of the object's structure and significantly facilitates the handling of dramatic shape variations in deformable objects, where $\mathbf{d}(p) \in R^{N}$ denotes the vector of the distance between the points and the skeleton. Please refer to \cite{5djsio} for detailed information.} 

After extracting key points from the input point cloud, the decoder $f_{dec}$, which is composed of $L$ multilayer perception (MLP), utilizes these key points to reconstruct the input point cloud by gradually augmenting finer geometric details in a hierarchical manner. Specifically, the $(L-i)$-th layer of the decoder can be written as
\begin{equation}
    \begin{aligned}
        \mathbf F_{dec}^{(L-i +1)} &= \mathbf h( \mathbf X^i_{enc}) \oplus \mathbf F_{enc}^i \\
        &\oplus \text{Proj}(\mathbf F_{dec}^{(L-i)}, \mathbf X_{enc}^{(i-1)}, \mathbf X_{enc}^{(i)})
    \end{aligned}
\end{equation}
where  $\mathbf X_{enc}^i\in \mathbb{R}^{N \times 3}$, $\mathbf F_{enc}^i \in \mathbb{R}^{N \times F_i}$ denotes the output features of the $i$-th encoder where $F_i$ denotes the corresponding dimensions, and $\text{Proj}(\mathbf F_{dec}^{(L-i)},\mathbf X_{enc}^{(i-1)},\mathbf X_{enc}^{(i)})$ is the feature projected from the former layer of the decoder with $\oplus$ being element-wise concatenation. Consequently, the loss function can be defined as
\begin{equation}
\label{loss}
    \mathcal{L} = \mathcal{L}_{sim} + \mathcal{L}_{far}
\end{equation}
where $\mathcal{L}_{sim}$ denotes the Chamfer distance between the input point cloud and the reconstructed point cloud, and $\mathcal{L}_{far}$ represents the Chamfer distance between the extracted key points and the points obtained by the farthest point sampling method, which ensures that the key points are evenly distributed on the surface of the object.

\subsection{Adaptive RBFNN Jacobian Estimator}
Among online learning methods, adaptive RBFNN has been widely studied due to its strong fitting and generalization ability. An RBFNN can be described as
\begin{equation}
    \boldsymbol{\phi}(\mathbf x) =\mathcolorbox{white} {\mathbf  W_c^T} \color{black}{\boldsymbol{\theta}}(\mathbf x)
\end{equation}
where $\mathcolorbox{white} {\mathbf W_c}$ represents the weight matrix, $\mathbf x$ denotes the input of the network and $\bm{\theta}(\cdot) = [\theta_1(\cdot), \theta_2(\cdot), \cdots, \theta_m(\cdot)]^T \in \mathbb{R}^{m}$ is the radial basis function. In this paper, the Gaussian radial function is utilized as the basis function
\begin{equation}
    \theta_i(\mathbf x) = \text{e}^{\frac{-\Vert \mathbf x - \boldsymbol \mu \Vert^2}{\sigma_i^2}}
\end{equation}
where $\boldsymbol \mu$ and $\sigma_i$ are the centers and width of the basis function, respectively. In this case, the input of the network can be designed as $\mathcolorbox{white}{\mathbf x = [\mathbf q^T , \mathbf p^T]^T \in \mathbb{R}^{12+3n}} $ where $\mathbf q$ represents the joint position of robots, and then the centers $ \boldsymbol \mu$ can be obtained by machine learning methods such as K-means, and the width $\sigma_i$ can be designed manually. If we consider the 3D coordinates of the key points extracted by Key-Grid in the perception part,  then we have
\begin{equation}
    \mathbf p = [\mathbf p_1, \mathbf p_2, \cdots,\mathbf p_n]^T \in \mathbb{R}^{3n}
\end{equation}
where $\mathbf p_i = [x_i, y_i, z_i]^T$ denotes the coordinates of the $i$-th key points. Then the Jacobian matrix $\mathbf J$ can be described as
\begin{equation}
    \label{jac}
 \dot{\mathbf p}=\mathbf J\dot{\mathbf q}=\left[\begin{matrix} \mathcolorbox{white}{ \dot{\mathbf p}_1}\\\vdots\\ \mathcolorbox{white}{\dot{\mathbf p}_n}\\\end{matrix}\right]=\left[\begin{matrix} \mathbf j_{11}&\cdots&\mathbf  j_{112}\\\vdots&\mathbf j_{ij}&\vdots\\\mathbf j_{n1}&\cdots&\mathbf j_{n12}\\\end{matrix}\right]\left[\begin{matrix}\dot{q}_1\\\vdots\\\dot{q}_{12}\\\end{matrix}\right]
\end{equation}
where  $\mathbf j_{ij}\in \mathbb{R}^{3}$ denotes the Jacobian for the $i$-th key point and $j$-th joint. Then we utilize RBFNN to approximate  $\mathbf j_{ij}$
\begin{equation}
\label{rbfest}
     \widehat{\mathbf j}_{ij}=\mathbf W_{ij}^T\bm{\theta}\left(\mathbf x\right) 
\end{equation}
 and the estimation error $ \widetilde{\mathbf j}_{ij}$ can be written as
 \begin{equation}
    \widetilde{\mathbf j}_{ij} = \widetilde{ \mathbf W}^T_{ij}\bm{\theta}(\mathbf x) + \boldsymbol \epsilon_{ij}
\end{equation}
where $ \widetilde{\mathbf W}_{ij} = \mathbf W_{ij}^* - \mathbf W_{ij}$ denotes the estimation errors of weight matrix $\mathbf W_{ij}$ with $\boldsymbol \epsilon_{ij}$ being a small bounded estimation errors of the RBFNN.
\subsection{Prescribed Performance Control}
Since the coordinates of the feature points are utilized as feature vectors, it becomes feasible to design the controller based on the spatial information embedded within these features. Consequently, we introduce PPC to constrain the control errors within predefined boundaries. By integrating the PPC, the feature points will be constrained to move within the desired boundaries, facilitating rapid convergence towards the target configuration. This combination enables the control system to effectively manage the dynamics of the deformable object, ensuring that deviations from the optimal trajectories are minimized. To design a PPC,  positive decreasing continuous functions are first defined for every element $e_i$ in the  error vector $\mathbf e_p = [\mathbf e_{p1}, \cdots, \mathbf e_{pn}]^T = [e_1, e_2, \cdots e_{3n-1}, e_{3n}]^T$ as
\begin{equation}
    \mu_i\left(t\right)=\left(\mu_{i0}-\mu_{i\infty}\right)\text{e}^{-\alpha_it}+\mu_{i\infty}
\end{equation}
where $\mu_{i0}>\mu_{i\infty}>0 $ denote maximum allowable error, steady-state error and $ \alpha_i>0$ is the convergence speed. Based on the performance function, the error boundaries can be defined as $\varphi_{ai}=-\delta_i\mu_i\left(t\right)$, $\varphi_{bi}=\delta_i\mu_i\left(t\right)$ which represent lower and upper boundaries respectively, with $\delta_i \in \mathbb{R}^+$ denoting a positive parameter selected by designer.

Combined with the defined boundaries, the errors can be converted into transfer errors as
\begin{equation}
\label{xi}
    \xi_i=S_i\left(e_i\right)\xi_{bi}+\left(1-S_i\left(e_i\right)\right)\xi_{ai}
\end{equation}
where $\boldsymbol \xi_a=\left[\begin{matrix}\frac{e_1}{\varphi_{a1}}, \cdots ,\frac{e_{3n}}{\varphi_{a3n}}\\\end{matrix}\right]$, $\boldsymbol \xi_b=\left[\begin{matrix}\frac{e_1}{\varphi_{b1}}, \cdots,\frac{e_{3n}}{\varphi_{b3n}}\\\end{matrix}\right]$, and $S_i\left(e_i\right) $ is defined as
\begin{equation}
    S_i(e_i) = \left\{
    \begin{aligned}
        1,         &\quad\quad e_i > 0\\
        0,         &\quad\quad e_i \leq 0
    \end{aligned}
    \right. .
\end{equation}
As such, the time derivative of $\xi_{ai}$ and $\xi_{bi}$ can be given by
\begin{equation}
\label{xiat}
    \dot{\xi}_{ai} = \frac{\dot{e}_i}{\varphi_{ai}} - \frac{e_i \dot{\varphi}_{ai}}{\varphi^2_{ai}},
\end{equation}
\begin{equation}
\label{xibt}
    \dot{\xi}_{bi} = \frac{\dot{e}_i}{\varphi_{bi}} - \frac{e_i \dot{\varphi}_{bi}}{\varphi^2_{bi}}.
\end{equation}
To facilitate the controller design process, a transient variable is defined as
\begin{equation}
\label{z}
    \mathbf z = \left[ \frac{\xi_1^2}{(1-\xi_1^2)e_1}, \frac{\xi_2^2}{(1-\xi_2^2)e_2}, \cdots, \frac{\xi_{3n}^2}{(1-\xi_{3n}^2)e_{3n}} \right]^T.
\end{equation}
Based on the transient variable and the visual errors, the kinematic controller is defined as follows:
\begin{equation}
    \label{outloop}
     \dot{\mathbf q} = - \widehat{\mathbf J}^{\dagger} [(\mathbf K_1 + \boldsymbol \eta )\mathbf e_p + \mathbf{K}_z \mathbf z]
\end{equation}
where $\mathbf{K}_1$ and $\mathbf {K}_z$ denote diagonal positive control gain matrices, $\mathbf{\widehat{J}^{\dagger}}$ represents the pseudo-inverse of the estimated Jacobian matrix. Furthermore, $\boldsymbol{\eta} = \text{diag}([\eta_1, \cdots, \eta_{3n}])$ denotes a time-varying gain matrix, where $\eta_i$ can be given by
\begin{equation}
    \eta_i = \sqrt{\left( \frac{\dot{\varphi}_{ai}}{\varphi_{ai}}\right)^2 +\left(\frac{\dot{\varphi}_{bi}}{\varphi_{bi}}\right)^2 + K_{\eta} }
\end{equation}
where $K_{\eta} \in \mathbb{R}^+$ represents a positive constant. Then, in order to eliminate the estimation errors, the adaptive law of $\mathbf W_{ij}$ can be designed as 
\begin{equation}
\label{adaptrbf}
    \dot{\mathbf W}_{ij}=\bm{\theta}\left( \mathbf x \right){\dot{q}}_j \mathbf z_{i}^T  - \gamma \mathbf W_{ij}
\end{equation}
where $\gamma \in \mathbb{R}^+$ represents a positive number. 

\subsection{Analysis of Lyapunov Stability}
Before proceeding with the Lyapunov stability proof, we first present an important lemma
\begin{lemma}\cite{5499019}
    For any positive constant $|v| < 1 $ and any positive integer $y$, one has
\end{lemma}
\begin{equation}
\label{lemma1}
    \text{ln}\frac{1}{1-v^{2y}} < \frac{v^{2y}}{1-v^{2y}}.
\end{equation}
\begin{theorem}
    For the DOM system illustrated in Fig. \ref{domp}, by employing the controller (\ref{outloop}) with the RBFNN estimator (\ref{rbfest}), whose weight matrices are updated according to the adaptive rules (\ref{adaptrbf}), it can be guaranteed that the closed-loop system is semiglobally uniformly ultimately bounded when the initial state of the system is bounded.
\end{theorem}

\begin{proof}

According to the definition of the error vector $\mathbf{e}_p$, we can obtain its time derivative as 
\begin{equation}
     \dot{\mathbf e}_p =  \dot{\mathbf p} - \dot{ \mathbf p}^*.
\end{equation}
Since the desired configuration of the deformable object is stationary, then one has $\dot{\mathbf e}_p = \dot{\mathbf p}$.
Substituting (\ref{jac}) into (\ref{outloop}), ${\dot{\mathbf e}}_p$ can be rewritten as
\begin{equation}
\label{e_p}
     \dot{\mathbf e}_p = -(\mathbf K_1 + \boldsymbol \eta) \mathbf e_p - \mathbf K_z \mathbf z + \Delta\dot{\mathbf p}_{net}.
\end{equation}
where $\Delta\dot{\mathbf p}_{net}$ can be written as
\begin{equation}
\Delta\dot{\mathbf p}_{net}=\mathbf J\dot{\mathbf q}-\widehat{\mathbf J}\dot{\mathbf q}=\left[\begin{matrix}\widetilde{\mathbf j}_{11}&\cdots&\widetilde{\mathbf j}_{112}\\\vdots&\widetilde{\mathbf j}_{ij}&\vdots\\\widetilde{\mathbf j}_{n1}&\cdots&\widetilde{\mathbf j}_{n12}\\\end{matrix}\right]\left[\begin{matrix}\dot{q}_1\\\vdots\\\dot{q}_{12}\\\end{matrix}\right]
\end{equation}
Define the following Lyapunov-like function candidate as
\begin{equation}
    V(t) = V_1(t) + V_2(t)
\end{equation}
where $V_1(t) =\displaystyle\sum^{3n}_{i=1} \left(\frac{\mathcolorbox{white}{S_i}}{2}\text{ln}\frac{1}{1-\xi_{bi}^2} + \frac{1-\mathcolorbox{white}{S_i}}{2}\text{ln}\frac{1}{1-\xi_{ai}^2} \right)$,  and $V_2(t) =\displaystyle\frac{1}{2} \sum_{i=1}^{12}\displaystyle\sum_{j=1}^{n} \text{tr}(\widetilde{\mathbf W}_{ij}^T\widetilde{\mathbf W}_{ij})$. Differentiating $V_1(t)$ with respect to time, then one has
\begin{equation}
\label{shangshi1}
    \dot{V}_1(t) = \sum^{3n}_{i=1}\left(\frac{S_i\xi_{bi}\dot{\xi}_{bi}}{1-\xi_{bi}^2} +\frac{(1-S_i)\xi_{ai}\dot{\xi}_{ai}}{1-\xi_{ai}^2} \right).
\end{equation}
According to the process 1) in Appendix, then one has
\begin{subequations}
\renewcommand{\theequation}{26} 
\begin{align}
    \dot{V}_1(t) &= \sum^{3n}_{i=1}\bigg[\frac{S_i\xi^2_{bi}}{(1-\xi^2_{bi})e_i}\bigg(\dot{e}_i - \frac{e_i\dot{\varphi}_{bi}}{\varphi_{bi}}\bigg) \bigg] \nonumber\\ 
    &+ \sum^{3n}_{i=1}\bigg[\frac{(1-S_i)\xi^2_{ai}}{(1-\xi^2_{ai})e_i}\bigg(\dot{e}_i - \frac{e_i\dot{\varphi}_{ai}}{\varphi_{ai}}\bigg) \bigg] \tag{26} \label{dv1}\\
    & = \sum^{3n}_{i=1}\bigg[\frac{\xi_i^2 \dot{e}_i}{(1-\xi_i^2)e_i} - \frac{S_i\xi^2_{bi}}{1-\xi^2_{bi}}\frac{\dot{\varphi}_{bi}}{\varphi_{bi}} - \frac{(1-S_i)\xi^2_{ai}}{1-\xi^2_{ai}} \frac{\dot{\varphi}_{ai}}{\varphi_{ai}}\bigg]\nonumber.
\end{align}
\end{subequations}
Then, introducing (\ref{z}) into (\ref{dv1}), one has
\begin{equation}
\label{dv12}
\begin{aligned}
    \dot{V}_1(t)= & \mathbf z^T\dot{\mathbf e}_p - \sum^{3n}_{i=1} \bigg[\frac{S_i\xi^2_{bi}}{1-\xi^2_{bi}}\frac{\dot{\varphi}_{bi}}{\varphi_{bi}} + \frac{(1-S_i)\xi^2_{ai}}{1-\xi^2_{ai}} \frac{\dot{\varphi}_{ai}}{\varphi_{ai}}\bigg] .
\end{aligned}
\end{equation}
Substituting (\ref{e_p}) into (\ref{dv12}), then one has
\begin{subequations}
\renewcommand{\theequation}{28} 
\begin{align}
    \dot{V}_1(t)& = -\mathbf z^T \mathbf K_z \mathbf z - \mathbf z^T(\mathbf K_1 + \boldsymbol \eta)\mathbf e_p +  \mathbf z^T \Delta \dot{\mathbf p}_{net} \nonumber\\
    & - \sum^{3n}_{i=1} \bigg[\frac{S_i\xi^2_{bi}}{1-\xi^2_{bi}}\frac{\dot{\varphi}_{bi}}{\varphi_{bi}} + \frac{(1-S_i)\xi^2_{ai}}{1-\xi^2_{ai}} \frac{\dot{\varphi}_{ai}}{\varphi_{ai}}\bigg] .
\end{align}
\end{subequations}

As such, it can be concluded that
\begin{subequations} 
\renewcommand{\theequation}{29} 
\begin{align}
    \dot{V}_1(t)& = -\mathbf z^T \mathbf K_z \mathbf z - \mathbf z^T(\mathbf K_1 + \boldsymbol\eta)\mathbf e_p  \nonumber\\
    & - \sum^{3n}_{i=1} \bigg[\frac{S_i\xi^2_{bi}}{1-\xi^2_{bi}}\frac{\dot{\varphi}_{bi}}{\varphi_{bi}} +\frac{(1-S_i)\xi^2_{ai}}{1-\xi^2_{ai}} \frac{\dot{\varphi}_{ai}}{\varphi_{ai}}\bigg]\nonumber\\ 
    & +  \sum^{12}_{i=1}\sum^{n}_{j=1} z_j(\widetilde{\mathbf W}_{ij}\bm{\theta}(\mathbf x)+\boldsymbol \epsilon_{ij})\dot{q}_i. \label{eq29}
\end{align}
\end{subequations}
According to the process 2) in Appendix, we have
\begin{subequations} 
\renewcommand{\theequation}{30} 
\begin{align}
\label{dv13}
    \dot{V}_1(t) &\leq \sum^{3n}_{i=1}\bigg(-\frac{\xi_i^2}{1-\xi_i^2} K_{1i}\bigg) - \mathbf z^T \mathbf K_z \mathbf z \nonumber\\
    & +\sum^{12}_{i=1}\sum^{n}_{j=1}z_j(\widetilde{\mathbf W}_{ij}\bm{\theta}(\mathbf x)+\boldsymbol\epsilon_{ij})\dot{q}_i.
\end{align} 
\end{subequations}
where $K_{1i}$ denotes the $i$-th element in the diagonal matrix $\mathbf K_1$.
The time derivation of $V_2(t)$ is
\begin{equation}
\label{v2t}
    \dot{V}_2(t) = -\sum_{i=1}^{12}\sum_{j=1}^{n}\text{tr}(\widetilde{\mathbf W}_{ij}\dot{\mathbf W}_{ij})
\end{equation}
Introducing the adaptive law (\ref{adaptrbf}) into (\ref{v2t}), then one has

\begin{equation}
\label{v22}
    \dot{V}_2(t) =  -\sum_{i=1}^{12}\sum_{j=1}^{n}\text{tr}(\widetilde{\mathbf W}_{ij}\bm{\theta}(\mathbf x)\dot{q}_iz_j^T) + \gamma\sum_{i=1}^{12}\sum_{j=1}^{n} \text{tr}(\widetilde{\mathbf W}_{ij}^T \mathbf W_{ij})
\end{equation}
Using (\ref{dv13}) and (\ref{v22}), it can be concluded that
\begin{equation}
\label{kuaile}
\begin{aligned}
    \dot{V}(t) &\leq \sum^{3n}_{i=1}\big(-\frac{\xi_i^2}{1-\xi_i^2} K_{1i}\big) -\mathbf K_z \mathbf z^T \mathbf z +\Vert \mathbf E \Vert \Vert\dot{\mathbf q}\Vert \mathbf z\Vert  \\& 
     - \gamma\sum_{i=1}^{12}\sum_{j=1}^{n} \Vert \mathbf W_{ij} \Vert^2 + \gamma\sum_{i=1}^{12}\sum_{j=1}^{n} \text{tr}(\widetilde{\mathbf W}_{ij}^T \mathbf W_{ij}) 
\end{aligned}
\end{equation}
where $\mathbf E$ denotes a matrix consisting of $\boldsymbol\epsilon_{ij}$.  Substituting (\ref{outloop}) into (\ref{kuaile}), then one has
\begin{equation}
    \begin{aligned}
        \dot{V}(t) &\leq \sum^{3n}_{i=1}\big(-\frac{\xi_i^2}{1-\xi_i^2} K_{1i}\big)- \mathbf K_z(1-\Vert \mathbf E \Vert)\Vert \mathbf z\Vert^2\\
        &- \frac{\gamma}{2}\sum_{i=1}^{12}\sum_{j=1}^{n} \Vert \mathbf W_{ij} \Vert^2 + \frac{\gamma}{2}\sum_{i=1}^{12}\sum_{j=1}^{n} \Vert \mathbf W_{ij}^* \Vert^2\\ & + \Vert \mathbf E \Vert\Vert\dot{\mathbf q}\Vert\Vert \mathbf z\Vert \leq -aV(t) + b
    \end{aligned}
\end{equation}
where $a=\text{min}(\bar{\sigma}(\mathbf K_1),\frac{\gamma}{2})$, $b=\Vert \mathbf E \Vert\Vert\dot{\mathbf q}\Vert\Vert \mathbf z\Vert + \frac{\gamma}{2}\sum_{i=1}^{12}\sum_{j=1}^{n} \Vert \mathbf W_{ij}^* \Vert^2$ and $\bar{\sigma}(\mathbf K_1)$ denotes the biggest eigenvalue of $\mathbf K_1$. According to the universal approximation theorem \cite{HORNIK1989359}, $\Vert \mathbf E \Vert$ should be extremely small, so we assume $1-\Vert \mathbf E \Vert>0$. Furthermore, combined with the assumption \ref{slow}, it can be concluded that $b =  k_b\Vert\boldsymbol z\Vert$, where $ k_b = \Vert \mathbf E \Vert\Vert\dot{\mathbf q}\Vert$ denotes a small positive number. When $\Vert \mathbf e_p(0) \Vert $ is bounded by $\overline{e}$, then it can be concluded $V(t)$ will converge to a small compact set
\begin{subequations}
\renewcommand{\theequation}{35} 
\begin{align}
    \label{set1}\Omega_v &= \{ V(t) \in \mathbb{R}^+ \mid V(t) \leq C_v \}, \\
    \quad C_v &= \frac{\gamma}{2a}\sum_{i=1}^{12}\sum_{j=1}^{n} \Vert \mathbf W_{ij}^* \Vert^2 + \frac{k_b}{a} \Vert z \Vert +V(0). \nonumber
\end{align}
\end{subequations}
Let $C_v^*> C_v, C_v^* \in R^+$ be a large enough constant such that the initial value of $V(t)$ is inside the compact set $\Omega_{v^*} = \{ V(t) \in R^+ | V(t) \leq C_v^* \}$. Therefore, $\mathbf e_p(0)$ and $\dot{V}(0)$ are bounded. If $\mathbf e_p(t)$ becomes unbounded, then $V(t)$ as well as $\dot{V}(t)$ will also become unbounded, and there must be a time instant $t_1$ such that
\begin{equation}
\label{conditionbianjie}
    1) \quad V(t_1) = C_v^* \quad \text{and}\quad 2) \quad \dot{V}(t_1)>0.
\end{equation}
However, according to (\ref{set1}), if $V(t_1)= C_v^* > C_v$, then $\dot{V}(t_1) < 0$, and this obviously contradicts (\ref{conditionbianjie}). Therefore, $\mathbf e_p(t)$ is ultimately bounded, as such $b$ is bounded \cite{8815930}. Therefore, the transferred visual tracking error term $\text{ln}\frac{1}{1-\xi_i^2}$ and the weight matrices of the neural networks will be bounded and convergent to a small compact set around zero. {Then, it can be obtained that $\left| \xi_i \right| \leq 1$, and $\varphi_a < e_i < \varphi_b$ can be guaranteed. Thus, we can get the conclusion that the tracking error is always remained within the boundaries during the control process, which improves the transient performance of the system and reduces the convergence time.}
\end{proof}

\begin{figure}[!t]\centering
	\includegraphics[width=\linewidth]{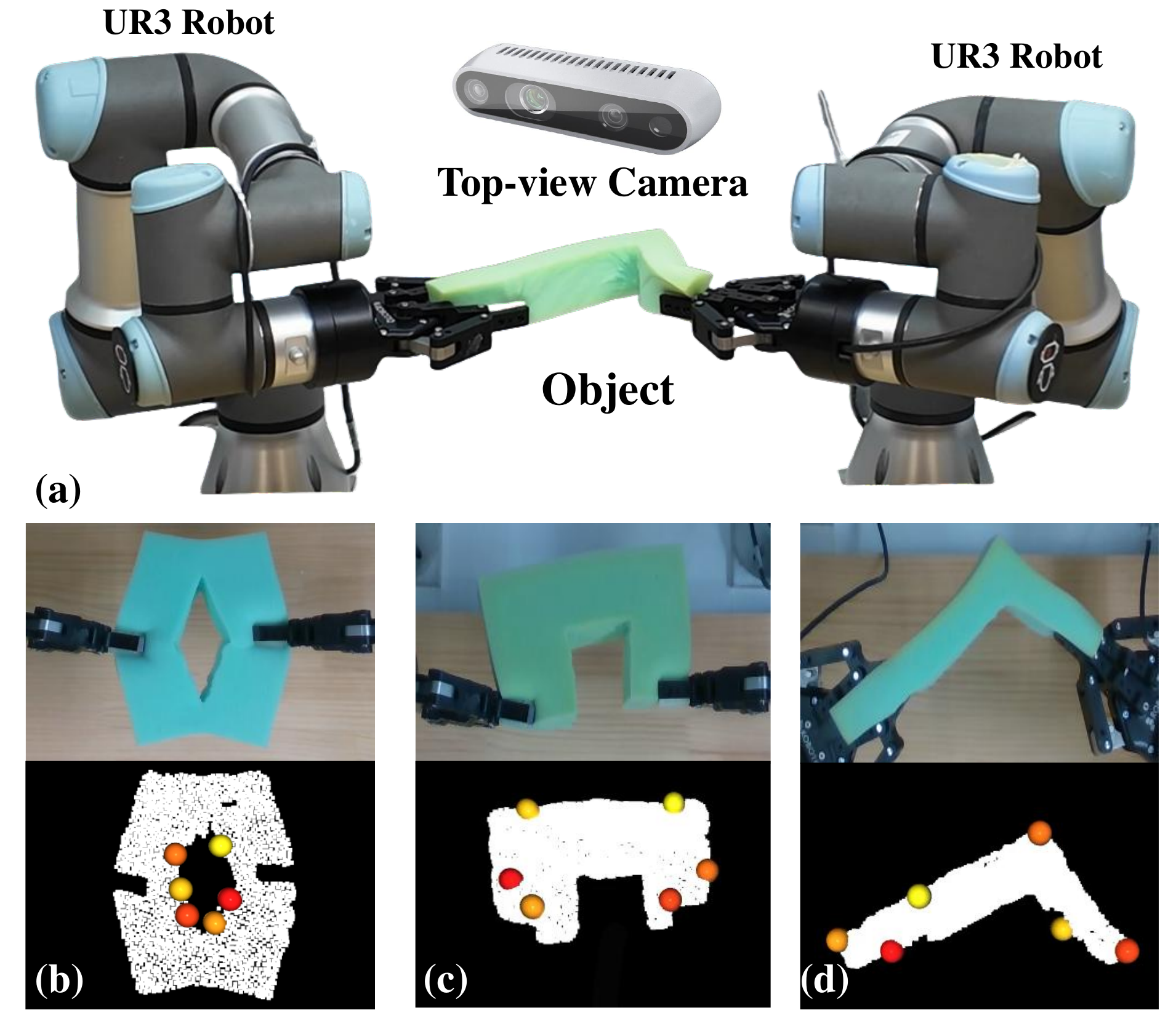}
	\caption{Experiment setup and the experimental objects. (a) the eye-to-hand robot-camera platform; (b) the object manipulated in task A (A piece of sponge with a slit incision in the middle, allowing it to be stretched) with the corresponding key points; (c) the object manipulated in task B (A sponge with a square incision) with the corresponding key points; (d) the object manipulated in task C (An ``L"-shape sponge) with the corresponding key points.}
\label{setup}
\end{figure}

\begin{table}[b]
    \centering
    \caption{Parameters of prescribed boundaries of different tasks.}
    \label{t1}
    \begin{tabularx}{0.4\textwidth} { 
  | >{\centering\arraybackslash}X 
  | >{\centering\arraybackslash}X 
  | >{\centering\arraybackslash}X 
  | >{\centering\arraybackslash}X | }
 ine
  & Task A & Task B & Task C \\
 ine
 $\mu_{x0}$  & 0.1  & 0.1 & 0.15  \\
 ine
 $\mu_{x\infty}$  & 0.01  & 0.015 & 0.015  \\
 ine
 $\alpha_x$  & 0.2  & 0.05 & 0.02 \\
 ine
 $\mu_{y0}$  & 0.1  & 0.15 & 0.15  \\
 ine
 $\mu_{y\infty}$  & 0.01  & 0.015 & 0.015 \\
 ine
 $\alpha_y$  & 0.2  & 0.05 & 0.02 \\
 ine
 $\mu_{z0}$  & 0.1  & 0.15 & 0.05  \\
 ine
 $\mu_{z\infty}$  & 0.01  & 0.015 & 0.01  \\
 ine
 $\alpha_z$  & 0.2  & 0.02 & 0.02 \\
 ine
\end{tabularx}
\end{table}

\section{Results}
\subsection{Experiment Setup}
Our experimental setup consisted of two Universal Robots 3 (UR3) and one Intel RealSense D435 camera. {For each task, we sampled 500 sets of point clouds for key-grid network training. In the data collection process and subsequent experiments, we used outlier filtering and farthest point sampling to  minimize the effect of environmental and sensing noise, and ensured that the number of points in each point cloud was fixed at 2048.} Then, We extracted 32 key points from each frame of point cloud data and selected six key points in the region of interest (ROI) as features for different tasks. 
Fig. \ref{setup} illustrates the setup for the experiments and the manipulation objects used in the three DOM tasks presented in this paper with the corresponding feature points selected in ROI. For all tasks, the controller parameters are set to $\mathbf K_1 = 2 \mathbf I_{18 \times 18}$, $\mathbf K_z = \mathbf I_{18 \times 18}$, $K_{\eta} = 0.5$. For each experiment, boundary conditions were designed for the $x$, $y$, and $z$ axes, which applied to all feature points and satisfied the following expression: $U_{bk} = -U_{ak} = (\mu_{k0 } - \mu_{k\infty})\text{e}^{-\alpha_k t} + \mu_{k\infty}$ where $k = \{ x, y, z\}$, the boundary conditions of experiments are listed in Table \ref{t1}. 

The Key-Grid network was first trained on a personal computer equipped with an NVIDIA RTX 3080Ti GPU with 12G memory and then deployed on a laptop equipped with an NVIDIA RTX 3060 GPU with 8 GB of memory to perform all experiments. Point cloud data of the sponges were captured using the top-view camera, and processed through OpenCV, and the neural network was constructed using PyTorch to extract the key points. {In addition, we sampled 100 sets of RBFNN input signals by performing random motions near the initial configuration, and then we obtained the centers of the basis function $\mu$ by the K-nearest-neighbor method.} The UR robot was controlled via ROS and the Urx library. To ensure the safety of the robots and avoid potential damage, the maximum joint speed of the UR robots was limited to 0.03 rad/s.

\subsection{Validation of Key Points Extraction}
To evaluate the effectiveness of the Key-Grid network, its performance is compared against PointNet++ \cite{10.5555/3295222.3295263}, Skeleton-Merger \cite{9578229}, SelfGeo \cite{zohaib2024selfgeo} and an learning-free method FPFH \cite{4650967}. We use the "L"-shaped sponge as an illustrative example to compare the key point extraction performance of various methods, as depicted in the Fig. \ref{compare_key}. The results demonstrate that Key-Grid achieves the best performance. Moreover, under deformation conditions, Key-Grid exhibits superior temporal and spatial consistency in the extracted key points.


\begin{figure}[h]\centering
	\includegraphics[width=\linewidth]{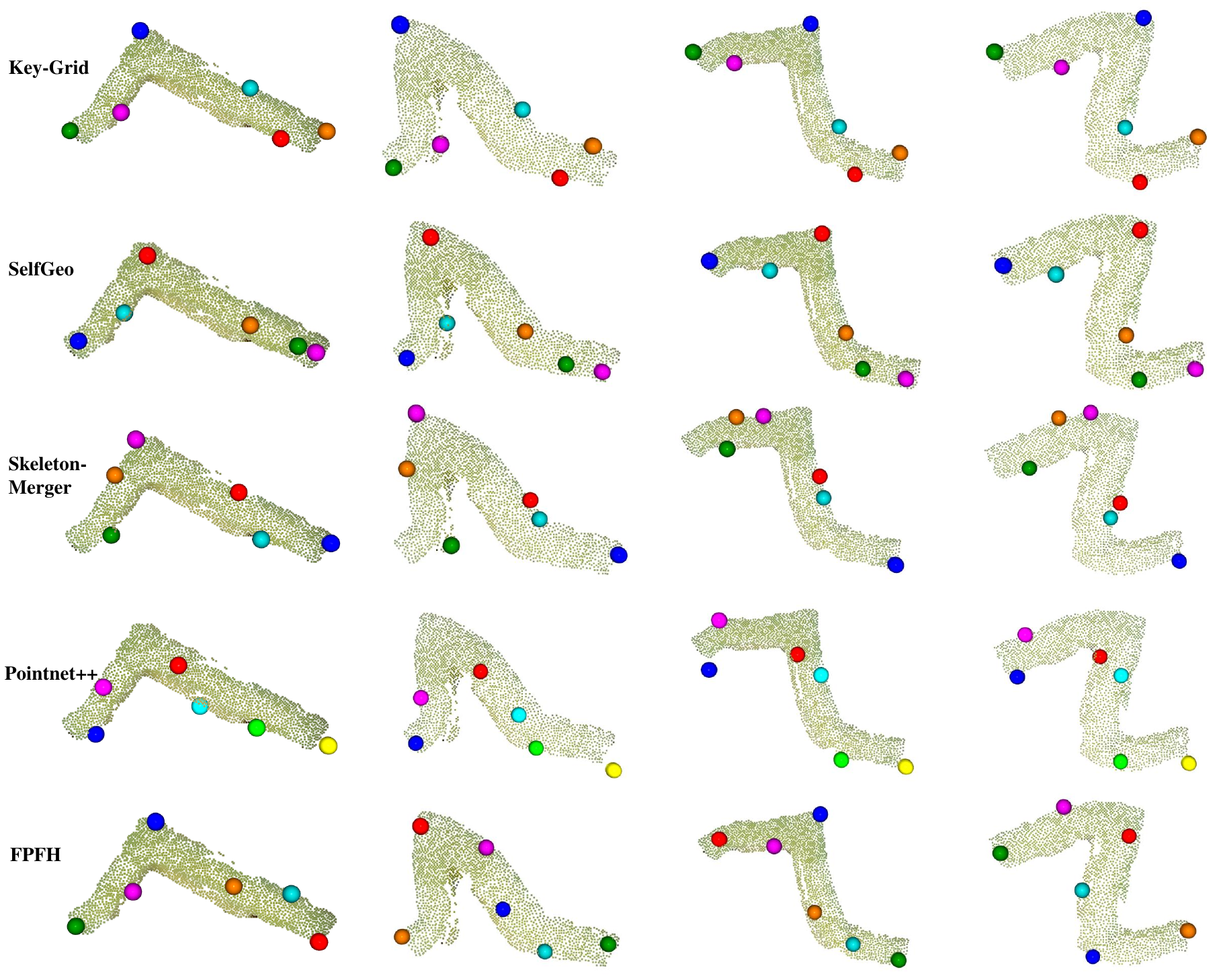}
	\caption{{Comparison of key point extraction. Compared with the other three learning-based methods, the key points extracted by Key-Grid have a more uniform distribution on the point cloud, and the key points have better spatiotemporal consistency. For the learning-free method FPFH, although the key points can be evenly distributed on the surface of the object, the corresponding spatial relationship is not preserved at all.}}
\label{compare_key}
\end{figure}

\begin{figure*}[]\centering
	\includegraphics[width=1\linewidth]{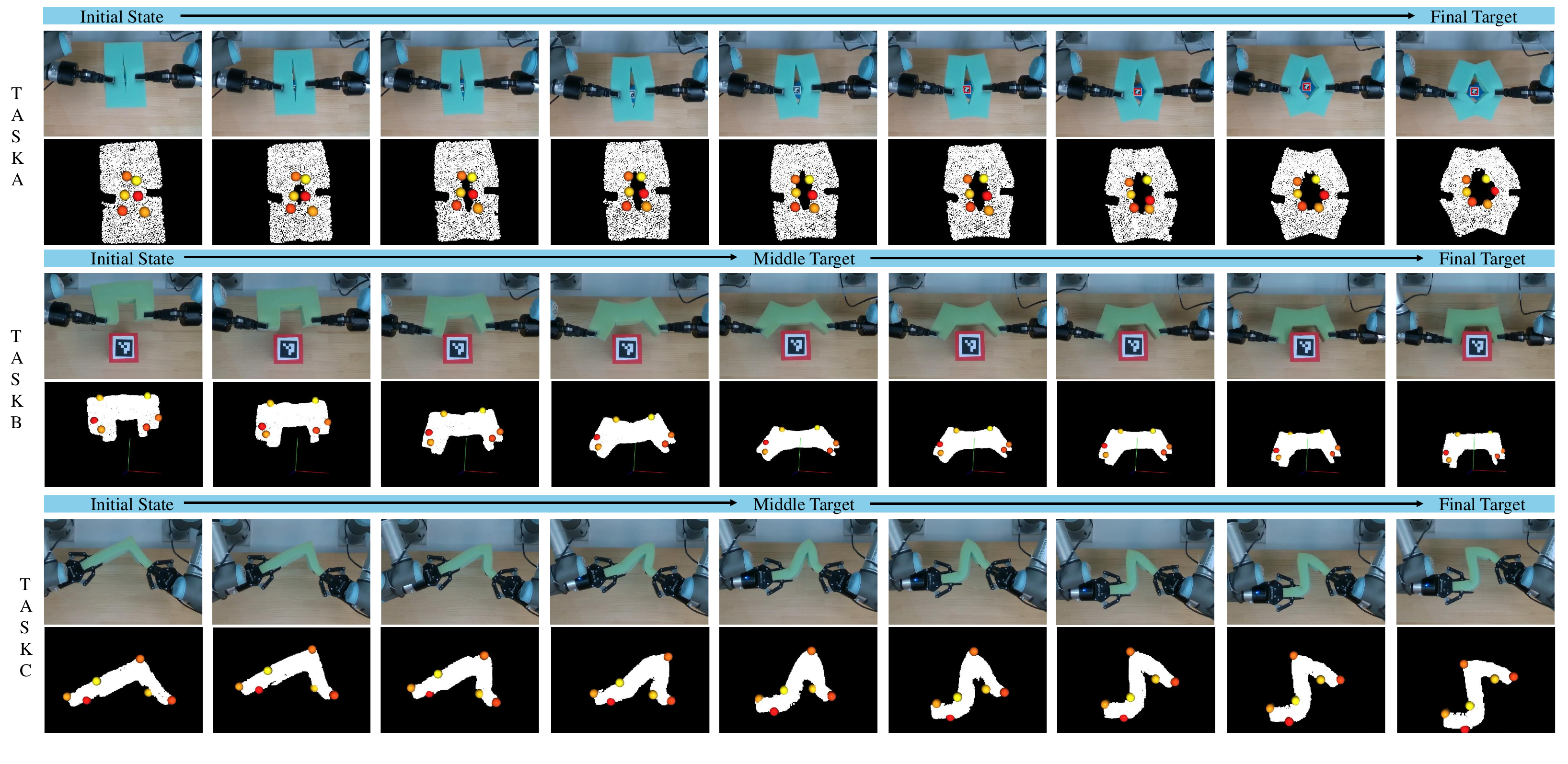}
	\caption{Qualitative results of three experiments. In Task A, the objective is to manipulate a sponge with a central incision, where the region surrounding the incision is designated as the ROI. Task B involves manipulating a sponge with a square incision at its center. The deformation of the areas on the left and right sides of the incision is significantly influenced by the robot's manipulation, and these areas are therefore defined as the ROI. Task C focuses on manipulating an "L"-shaped sponge to achieve different letter-shaped configurations, with the entire surface of the sponge designated as the ROI. To ensure uniform feature representation, key points are distributed evenly across the sponge's surface. For enhanced control accuracy, Tasks B and C are further divided into two consecutive DOM phases, each with distinct target configurations.}
\label{exp_all}
\end{figure*}

\subsection{Experiments}
The target of task A is to manipulate a sponge with an incision in the middle as shown in Fig. \ref{setup}(b), and to open the incision to a size large enough so that the camera can detect the Aruco code hidden under the sponge. The experimental process is shown in Fig. \ref{exp_all}, and the total position errors, as illustrated in Fig. \ref{cp1}(a), decreases steadily over time, reflecting the effectiveness of the control strategy in achieving DOM. The 3D trajectory visualization in Fig. \ref{cp1}(b) shows that the key points tend to move along straight paths toward the target positions, indicating a smooth and efficient control process. Furthermore, Fig. \ref{cp1}(c) shows that the position tracking errors in the X, Y, and Z axes remain within the preset boundaries. These results verify the reliability of the proposed control framework.
\begin{figure}[!t]\centering
	\includegraphics[width=\linewidth]{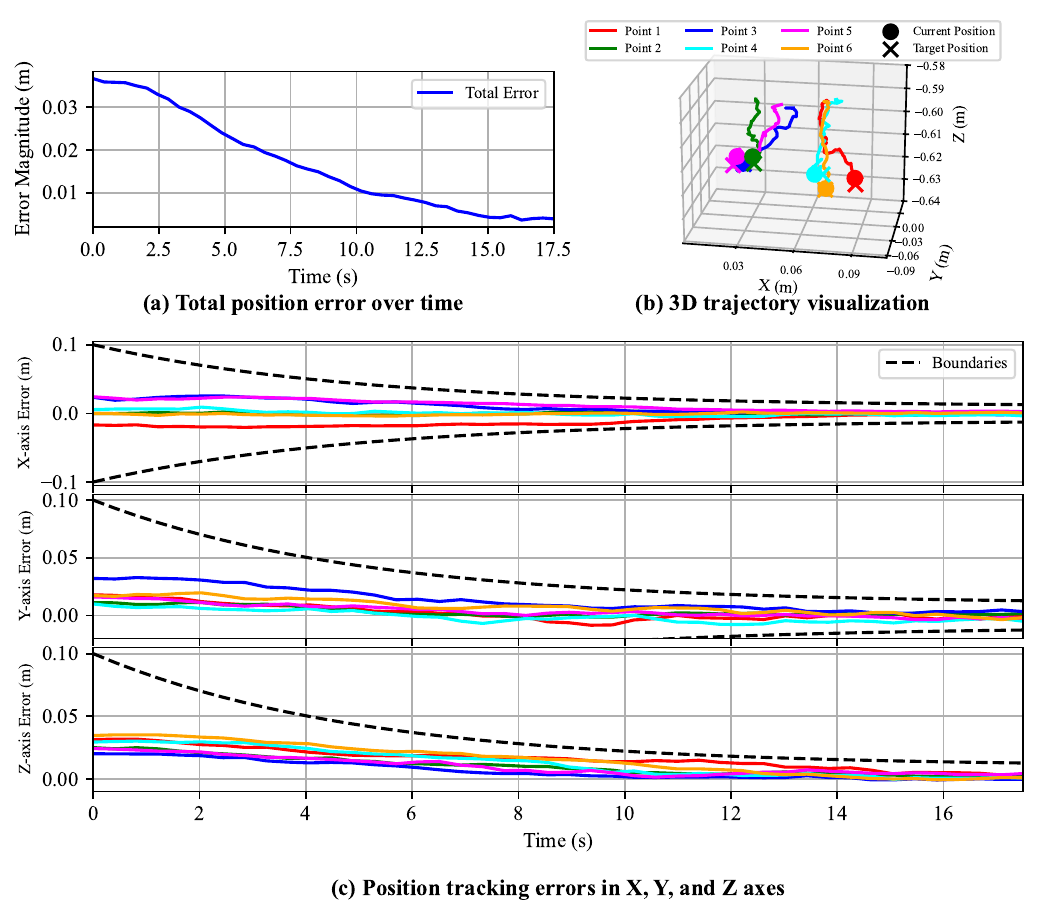}
	\caption{{Control performance of task A: (a) The curve of the norm of error vector $\Vert e_p \Vert$, (b) The trajectories of the key points, (c) The errors and corresponding boundaries of key points on x, y, and z-axes.}}
\label{cp1}
\end{figure}

The target of task B is to manipulate a sponge with a square incision in the middle, as shown in Fig. \ref{setup}(c), to grab a rigid rectangular block. To accomplish this, the control task is divided into two stages: the incision opening stage and the grabbing stage, with a five-second interval between the two stages for clear differentiation, as depicted in Fig. \ref{exp_all}. {It is worth mentioning that for the two different stages, we set different targets respectively, and the preset performance boundaries of the second stage will be reset when stage two begins.} The control performance is shown in Fig. \ref{cp2}. The total error curve in Fig. \ref{cp2}(a) shows a significant reduction during both stages, reflecting effective control. In Fig. \ref{cp2}(b), the key points exhibit smooth trajectories, progressing systematically toward their middle and final target positions. Fig. \ref{cp2}(c) demonstrates that the position tracking errors in the X, Y, and Z axes remain within acceptable boundaries throughout the process, verifying the accuracy and robustness of the control strategy. It should be noted that, due to the large initial errors on the z-axis and the limited joint speed of the robots, the error boundary of the z-axis was designed to be looser.
\begin{figure}[!t]\centering
	\includegraphics[width=\linewidth]{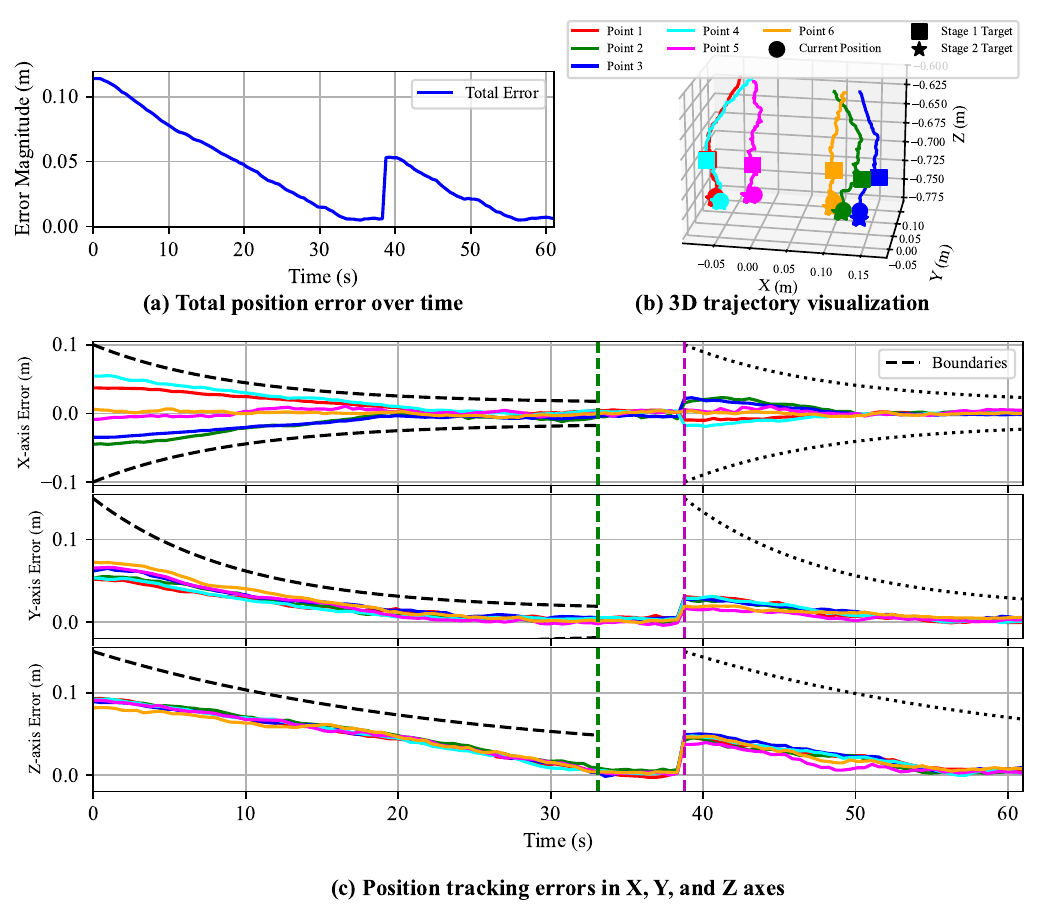}
	\caption{{Control performance of task B in two stages: (a) The curve of the norm of error vector $\Vert e_p \Vert$, (b) The trajectories of the key points, (c) The errors and corresponding boundaries of key points on x, y, and z-axes.}}
\label{cp2}
\end{figure}

Task C is to manipulate an "L"-shaped sponge to form different letter shapes, as illustrated in Fig. \ref{setup}(d). This task is divided into two stages, separated by a five-second interval for clarity. The objective of the first stage is to form an asymmetrical "V" shape, while the second stage aims to create an "S" shape, as shown in Fig. \ref{exp_all}. Same as Task B, the two stages have different target configurations and independent preset boundaries. The control performance is presented in Fig. \ref{cp3}. In Fig. \ref{cp3}(a), the total error decreases significantly in both stages, indicating effective control. Fig. \ref{cp3}(b) shows smooth and coordinated movements of the key points, resulting in the desired shapes. Fig. \ref{cp3}(c) demonstrates that the tracking errors along the X, Y, and Z axes remain within the preset boundaries, validating the efficacy of our proposed method.
\begin{figure}[t]\centering
	\includegraphics[width=\linewidth]{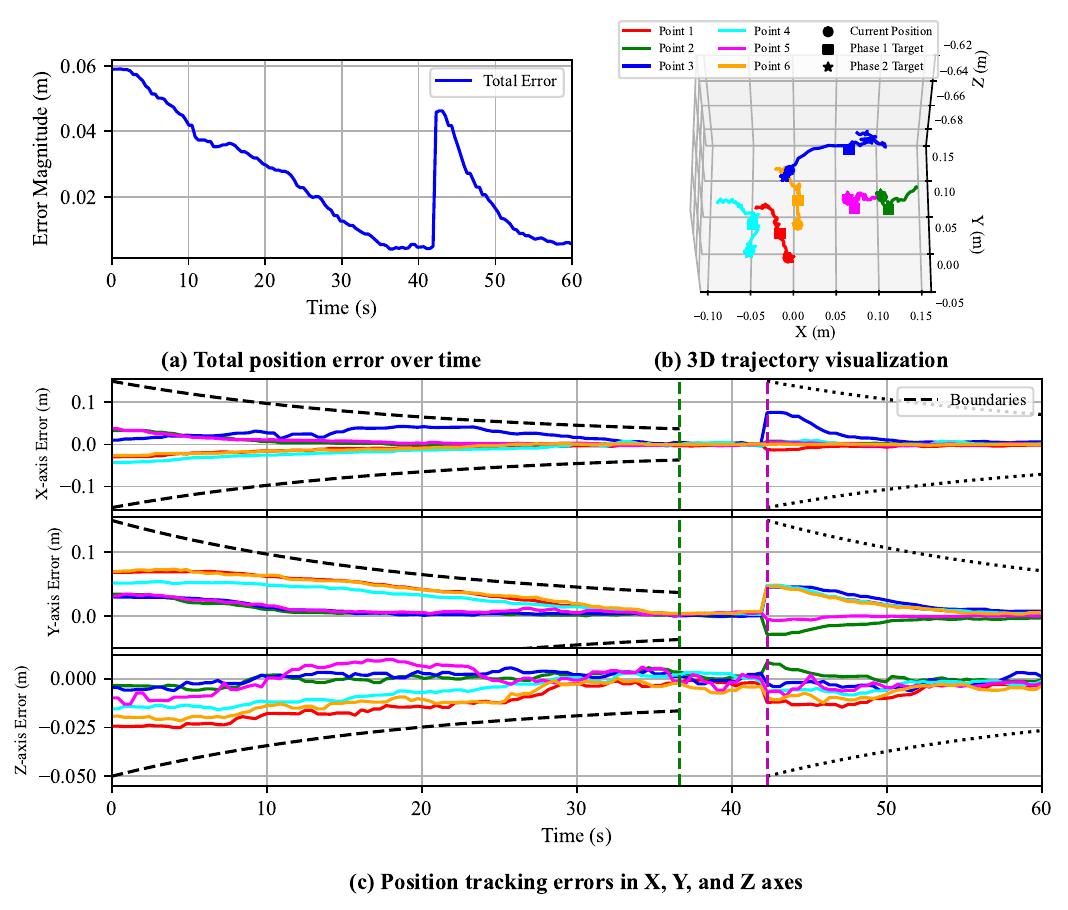}
	\caption{{Control performance of task B in two stages: (a) The curve of the norm of error vector $\Vert e_p \Vert$, (b) The trajectories of the key points, (c) The errors and corresponding boundaries of key points on x, y, and z-axes.}}
\label{cp3}
\end{figure}

\subsection{Comparative Study}

{To evaluate the effectiveness of the proposed algorithm, we conducted a comparative analysis against two existing key points-based methods: the GMLC presented in \cite{9888782}, which utilized manually marked key points and designed an adaptive optimization-based controller, and the G-DOOM method introduced in \cite{9811597}, which extracted keypoints from depth images and designed a graph network-based MPC controller. Considering the needs of ablation experiments, we only compare three control methods, and the key points are uniformly extracted using Key-Grid.} For each of the three distinct tasks, 20 trials were conducted for each method to ensure a comprehensive evaluation. The specific data for steady-state errors, convergence times, and success rates are presented in Table \ref{t3}, providing a detailed quantitative evaluation of the proposed method. Additionally, both quantitative and qualitative analyses are illustrated in Fig. \ref{compare}, offering a comprehensive comparison of performance across different approaches. From these results, it is evident that the proposed method consistently achieves smaller steady-state errors, faster convergence times, and higher success rates across all three tasks. {This is because the optimization-based controller in GMLC and the MPC controllers in G-DOOM can not to improve the transient performance of the system through the design of dynamic constraints. In contrast, the PPC controller proposed in this paper addresses this limitation by incorporating boundary constraints, which effectively enhances the system's transient performance.} Furthermore, these comparison between the success rates highlight the robustness of the proposed method in DOM tasks, effectively handling complex deformable dynamics and outperforming existing methods under various conditions.
\begin{table}[ht]
    \centering
    \caption{Performance of different control methods.}
    \label{t3}
    \begin{tabular*}{0.487\textwidth}{| >{\raggedright\arraybackslash}p{0.12\textwidth} 
                                         | >{\centering\arraybackslash}p{0.09\textwidth} 
                                         | >{\centering\arraybackslash}p{0.09\textwidth} 
                                         | >{\centering\arraybackslash}p{0.09\textwidth} |}
        ine
        & Ours & GLMC & G-DOOM \\
        ine
        Task A error (cm)  & \textbf{1.3 $\pm$ 0.4 }  & 1.7 $\pm$ 0.3   &  1.8 $\pm$ 0.4   \\
        ine
        Task A Time (s)& \textbf{22 $\pm$ 4 }  & 28 $\pm$ 6  &  30 $\pm$ 5  \\
        ine
        Task A rate  & \textbf{100\% }  & 95\%  &  90 \% \\
        ine
        Task B error (cm)  & \textbf{1.9 $\pm$ 0.4 }  & 2.3 $\pm$ 0.4   &  2.2 $\pm$ 0.2   \\
        ine
        Task B Time (s)& \textbf{65 $\pm$ 4 }  & 76 $\pm$ 4  &  79 $\pm$ 5  \\
        ine
        Task B rate  & \textbf{ 85\% }  & 75\%  &  70\%  \\
        ine
        Task C error (cm)  & \textbf{1.5 $\pm$ 0.3 }  & 2.1 $\pm$ 0.4   &  2.5 $\pm$ 0.4   \\
        ine
        Task C Time (s)& \textbf{64 $\pm$ 4 }  & 75 $\pm$ 5  &  77 $\pm$ 6  \\
        ine
        Task C rate  & \textbf{85 \%}  & 70 \%  &  55\%  \\
        ine
    \end{tabular*}
\end{table}

\begin{figure}[t]\centering
	\includegraphics[width=\linewidth]{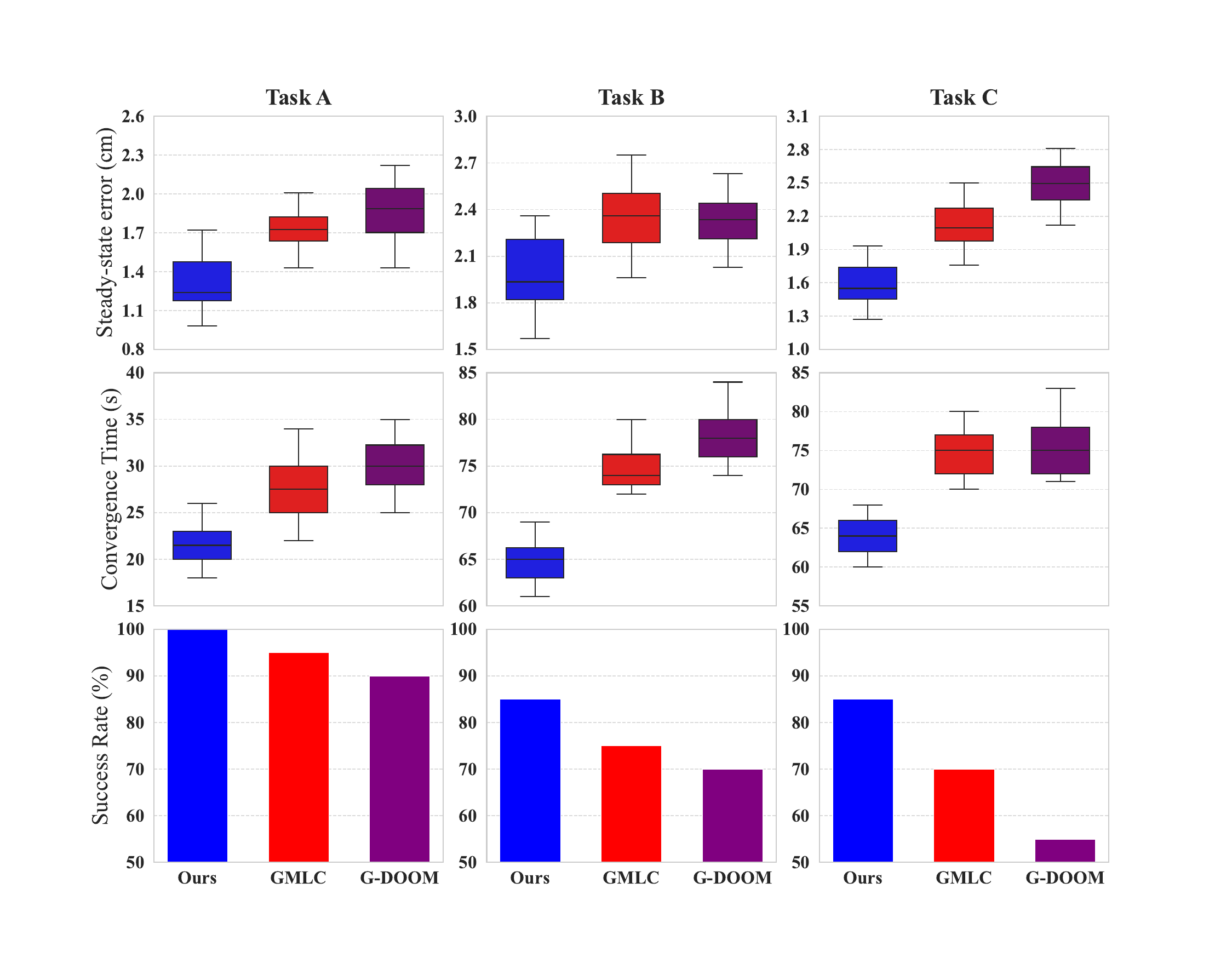}
	\caption{ Comparison of our proposed shape control model against other two methods \cite{9888782} and \cite{9811597}, the three rows from top to bottom represent the steady-state errors, convergence time, and success rates respectively, and the three columns represent three types of tasks. Compared with the other two methods, our proposed method performs better in all three indicators, showing that the proposed method has better steady-state, transient performance and robustness.}
\label{compare}
\end{figure}

\section{Conclusion}
This paper introduces a prescribed performance control method for the manipulation of deformable objects in a latent space that encapsulates spatial information. The proposed approach first extracts the coordinates of key points from the point cloud of the deformable object and represents them as feature vectors using the Key-Grid neural network. By leveraging this representation, which effectively preserves the spatial structure of the object while reducing the dimensionality of the feature space, a prescribed performance controller is designed to perform the manipulation process. The controller ensures that the errors of the key points converge within a predefined performance boundary, thereby improving precise control and enhanced performance. 

The efficacy of the proposed method is rigorously validated through three sets of comparative experiments. The experimental results consistently demonstrate that the proposed method achieves superior steady-state and transient performance when compared to two state-of-the-art approaches. Furthermore, the method exhibits a significantly higher robustness under diverse manipulation scenarios, further underscoring its potential for practical applications in deformable object manipulation tasks. {However, the main limitation of this method is that it is not robust to occlusion. Future work will focus on solving the occlusion problem with this method by combining deep learning methods such as temporal neural networks and generative adversarial networks.}

{\appendix
{1) Detailed process for equation (\ref{dv1}):

Substituting (\ref{xiat}) and (\ref{xibt}) into (\ref{shangshi1}), then one has}
\begin{equation}
\begin{aligned}
    \mathcolorbox{white}{\dot{V}_1(t)} & \mathcolorbox{white}{= \sum^{3n}_{i=1}\bigg[\frac{S_i\xi_{bi}}{(1-\xi^2_{bi})}\frac{1}{\varphi_{bi}}\bigg(\dot{e}_i - \frac{e_i\dot{\varphi}_{bi}}{\varphi_{bi}}\bigg) \bigg]}\\
    & \mathcolorbox{white}{+ \sum^{3n}_{i=1}\bigg[\frac{(1-S_i)\xi_{ai}}{(1-\xi^2_{ai})}\frac{1}{\varphi_{bi}}\bigg(\dot{e}_i - \frac{e_i\dot{\varphi}_{ai}}{\varphi_{ai}}\bigg) \bigg]} \nonumber
\end{aligned}
\end{equation}
 {Due to} $\mathcolorbox{white}{\xi_{ai} = \frac{e_i}{\varphi_{ai}}}$, $\mathcolorbox{white}{\xi_{bi} = \frac{e_i}{\varphi_{bi}}}$, one has
\begin{equation}
\label{m1}
\begin{aligned}
    \mathcolorbox{white}{\dot{V}_1(t)} &\mathcolorbox{white}{= \sum^{3n}_{i=1}\bigg[\frac{S_i\xi^2_{bi}}{(1-\xi^2_{bi})e_i}\bigg(\dot{e}_i - \frac{e_i\dot{\varphi}_{bi}}{\varphi_{bi}}\bigg) \bigg]}\\ 
    &\mathcolorbox{white}{+ \sum^{3n}_{i=1}\bigg[\frac{(1-S_i)\xi^2_{ai}}{(1-\xi^2_{ai})e_i}\bigg(\dot{e}_i - \frac{e_i\dot{\varphi}_{ai}}{\varphi_{ai}}\bigg) \bigg]}. \nonumber
\end{aligned}
\end{equation}
{Then, introducing (\ref{xi}) into the above equation, one has}
\begin{equation}
\begin{aligned}
    \mathcolorbox{white}{\dot{V}_1(t) = \sum^{3n}_{i=1}\bigg[\frac{\xi_i^2 \dot{e}_i}{(1-\xi_i^2)e_i} - \frac{S_i\xi^2_{bi}}{1-\xi^2_{bi}}\frac{\dot{\varphi}_{bi}}{\varphi_{bi}} - \frac{(1-S_i)\xi^2_{ai}}{1-\xi^2_{ai}} \frac{\dot{\varphi}_{ai}}{\varphi_{ai}}\bigg].} \nonumber
\end{aligned}
\end{equation}

{2) Detailed process for equation (\ref{dv13}):

According to (\ref{xi}) and (\ref{z}), it can be concluded that for} $\mathcolorbox{white}{S_i = 1}$, $\mathcolorbox{white}{\eta_iz_ie_i = \frac{\eta_i\xi_{b_i}^2}{1-\xi_{bi}^2}}$, {then one has}
\begin{equation}
    \mathcolorbox{white}{-\eta_iz_ie_i-\frac{\xi_{bi}^2}{1-\xi_{bi}^2}\frac{\dot{\varphi}_{bi}}{\varphi_{bi}} = -\frac{\xi_{bi}^2}{1-\xi_{bi}^2}(\eta_i+\frac{\dot{\varphi}_{bi}}{\varphi_{bi}}) \leq 0}. \nonumber
\end{equation}
{Similarly, for} $\mathcolorbox{white}{S_i = 0}$, {we have}
\begin{equation}
    \mathcolorbox{white}{-\eta_iz_ie_i-\frac{\xi_{ai}^2}{1-\xi_{ai}^2}\frac{\dot{\varphi}_{ai}}{\varphi_{ai}} = -\frac{\xi_{ai}^2}{1-\xi_{ai}^2}(\eta_i+\frac{\dot{\varphi}_{ai}}{\varphi_{ai}}) \leq 0} \nonumber.
\end{equation}
{As such, it can be concluded that}
\begin{equation}
    \mathcolorbox{white}{-\mathbf{z}^T\boldsymbol\eta\mathbf e_p- \sum^{3n}_{i=1} \bigg[\frac{S_i\xi^2_{bi}}{1-\xi^2_{bi}}\frac{\dot{\varphi}_{bi}}{\varphi_{bi}} +\frac{(1-S_i)\xi^2_{ai}}{1-\xi^2_{ai}} \frac{\dot{\varphi}_{ai}}{\varphi_{ai}}\bigg] \leq 0}. \nonumber
\end{equation}
{Therefore, by substituting the above formula into (\ref{eq29}), we can obtain (\ref{dv13}).}
}

\bibliographystyle{IEEEtran}
\bibliography{reference}\ 

\vfill

\end{document}